\documentclass[10pt,twocolumn,letterpaper]{article}

\usepackage{cvpr}
\usepackage{times}
\usepackage{epsfig}
\usepackage{graphicx}
\usepackage{amsmath}
\usepackage{amssymb}
\usepackage[linesnumbered,ruled]{algorithm2e}
\usepackage[pagebackref=true,breaklinks=true,letterpaper=true,colorlinks,bookmarks=false]{hyperref}
\usepackage{multirow}
\usepackage{booktabs}
\usepackage{amsthm}

\newcommand{\eat}[1]{}

\DeclareMathOperator*{\argmin}{arg\,min}
\DeclareMathOperator*{\argmax}{arg\,max}

\newtheorem{lemma}{Lemma}

\cvprfinalcopy 

\def\cvprPaperID{2271} 

\ifcvprfinal\fi
\begin{document}

\title{Deep Metric Learning via Facility Location}

\newcommand*{\affmark}[1][*]{\textsuperscript{#1}}
\newcommand*{\email}[1]{\texttt{#1}}

\author{%
Hyun Oh Song\affmark[1], Stefanie Jegelka\affmark[2], Vivek Rathod\affmark[1], and Kevin Murphy\affmark[1] \vspace{0.5em}\\
{\affmark[1]Google Research, ~\affmark[2]MIT} \vspace{0.2em} \\
{\affmark[1]\email{\{hyunsong,rathodv,kpmurphy\}@google.com}, ~\affmark[2]\email{stefje@csail.mit.edu}}\\
}

\maketitle

\begin{abstract}
Learning image similarity metrics in an end-to-end
fashion with deep networks has demonstrated excellent results on tasks
such as clustering and retrieval.
However, current methods,
all focus on a very local view of the data.
In this paper, we propose a new metric learning scheme,
based on structured prediction,
that is aware of the global structure of the embedding
space, and which is designed to optimize
a clustering quality metric (NMI).
We show state of the art performance
on standard datasets, such as 
CUB200-2011 \cite{birds}, Cars196 \cite{cars}, and
Stanford online products \cite{liftedstruct} on NMI and R@K evaluation metrics.
\end{abstract}

\eat{
\begin{abstract}
Learning the representation and the similarity metric in an end-to-end
fashion with deep networks have demonstrated outstanding results
\cite{facenet, seanbell, liftedstruct, npairs} for clustering and
retrieval. However, these recent approaches still suffer from the
performance degradation stemming from the local metric training
procedure which is unaware of the global structure of the embedding
space. 

We propose a global metric learning scheme for optimizing the deep
metric embedding with the learnable clustering function and the
clustering metric (NMI) in a novel structured prediction framework. 

Our experiments on CUB200-2011 \cite{birds}, Cars196 \cite{cars}, and
Stanford online products \cite{liftedstruct} datasets show state of
the art performance both on the clustering and retrieval tasks
measured in the NMI and Recall@K evaluation metrics. 
\end{abstract}
}

\section{Introduction}
\label{sec:intro}
Learning to measure the similarity among arbitrary groups of data is
of great practical importance,
and can be used for a variety of tasks such as feature
based retrieval \cite{liftedstruct}, clustering \cite{deepclustering},
near duplicate detection \cite{Zheng_2016_CVPR}, verification
\cite{signatureVerification, faceVerification}, feature matching
\cite{universal_correspondence}, domain adaptation
\cite{domaintransduction}, video based weakly supervised learning \cite{triplet_video}, etc.
Furthermore, metric learning can be used for 
challenging extreme classification settings \cite{extreme_varma,
  extreme_langford, extreme_dhillon}, where the number of classes is
very large and the number of examples per class becomes scarce.
For example,
\cite{seanbell} uses this approach to perform
product search with $10$M images,
and 
\cite{facenet} shows superhuman performance on face
verification with 260M images of 8M distinct identities. 
In this setting, any direct classification or regression methods become
impractical due to the prohibitively large size of the label set.

Currently, the best approaches to metric learning
employ state of art neural networks \cite{alexnet, googlenet,
  vgg, resnet}, which are trained to produce an embedding of each
input vector so that a certain loss, related to distances of the
points, is minimized.
However, most current methods,
such as  \cite{facenet, seanbell, liftedstruct, npairs},
are very myopic in the sense that the loss is defined in
terms of  pairs or triplets inside the training mini-batch.
These methods
don't
take the global structure of the embedding space into
consideration, which can result in reduced  clustering
and retrieval performance.

Furthermore, most of the current methods \cite{facenet, seanbell, liftedstruct, npairs} in deep metric learning require a separate data preparation stage where the training data has to be first prepared in pairs \cite{contrastive, seanbell}, triplets \cite{triplet, facenet}, or n-pair tuples \cite{npairs} format. This procedure has very expensive time and space cost as it often requires duplicating the training data and needs to repeatedly access the disk.

In this paper,  we propose a novel learning framework which encourages
the network to learn an embedding function that directly 
optimizes a clustering quality metric (We use the normalized mutual information or  NMI metric
\cite{manningbook} to measure clustering quality,
but other metrics could be used instead.) and doesn't require the training data to be preprocessed in rigid paired format.
Our approach uses a structured prediction framework
\cite{Tsochantaridis/etal/04, Joachims/etal/09a}
to ensure that the score of the ground truth clustering assignment 
is higher than the score of any other clustering assignment.
Following the evaluation protocol in \cite{liftedstruct}, we report
state of the art results on CUB200-2011 \cite{birds}, Cars196
\cite{cars}, and Stanford online products \cite{liftedstruct} datasets
for clustering and retrieval tasks.

\section{Related work}
\label{sec:related}

\begin{figure*}[thbp]
\centering
\includegraphics[width=0.95\textwidth]{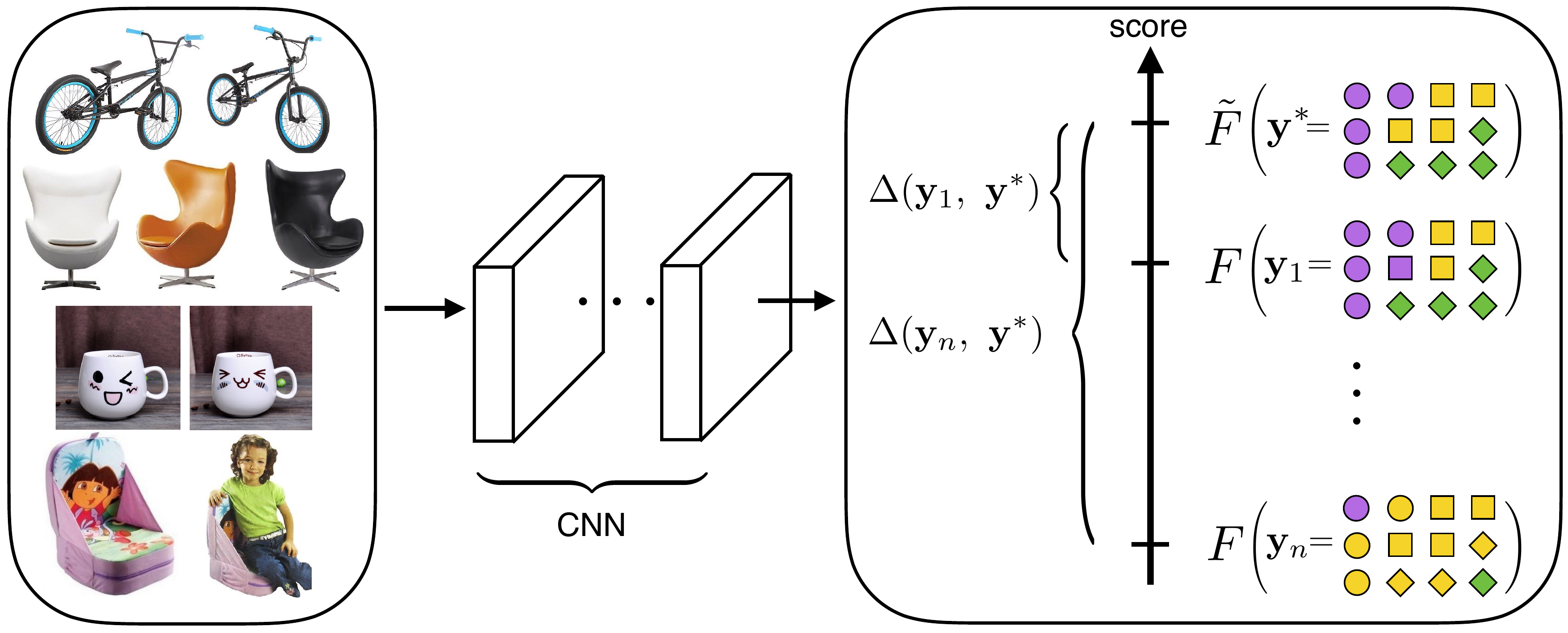}
\caption{Overview of the proposed framework. The network first computes the embedding vectors for each images in the batch and learns to rank the clustering score $\tilde{F}$ for the ground truth clustering assignment higher than the clustering score $F$ for any other assignment at least by the structured margin $\Delta(\mathbf{y},~\mathbf{y}^*)$.} 
\label{fig:figure1}
\end{figure*}

The seminal work in deep
metric learning is to train a siamese network
with contrastive loss
\cite{contrastive, faceVerification} where the task is to minimize the
pairwise distance between a pair of examples with the same class
labels, and to push the pairwise distance between a pair of examples
with different class labels at least greater than some fixed margin.

One downside of this approach is that it focuses on  absolute
distances, whereas for most tasks, relative distances matter more.
For this reason, more recent methods have proposed different loss
functions. We give a brief review of these below,
and we compare our method to them experimentally in Section~\ref{sec:results}.

\subsection{Triplet learning with semi-hard negative mining}

One improvement over contrastive loss is to use
triplet loss
\cite{triplet, triplet-svm}.
This first constructs a set of triplets,
where each triplet has
an \emph{anchor}, a \emph{positive}, and a \emph{negative} example,
where the anchor and
the positive have the same class labels and the negative has the
different class label.
It then tries to move the anchor and positive
closer than the distance between the anchor and the negative with some
fixed margin.
More precisely, it minimizes the following loss:
\begin{align}
  \ell(X,y) &= \frac{1}{|\mathcal{T}|} \sum_{(i,j,k) \in \mathcal{T}}
  \left[ D^2_{i,j} + \alpha -  D^2_{i,k} \right]_+
\end{align}
where
$\mathcal{T}$ is the set of triples,
$D_{i,j} =|| f(X_i) - f(X_j) ||_2$ is the Euclidean
distance in embedding space,
the operator $[\cdot]_+$ denotes the hinge function
which takes the positive component of the argument,
and $\alpha$ denotes a fixed
margin constant. 

In practice, the performance of these methods depends highly on
the  triplet sampling strategy. FaceNet \cite{facenet}
suggested the following online hard negative mining strategy.
 The idea is to construct
 triplets by associating with each positive pair in the minibatch
 a ``semi-hard'' negative example.
 This is an example which is further
 away from the anchor $i$ than the positive exemplar $j$ is,
 but still hard because the distance is close to the $i-j$ distance.
More precisely, it minimizes
\[
  \ell\left(X, \mathbf{y}\right)
  = \frac{1}{|\mathcal{P}|}\sum_{\left(i,j\right) \in \mathcal{P}}
  \left[ D_{i,j}^2 + \alpha - D^2_{i,k^*\left(i, j\right)} \right]_+
\]
where
\[
k^*\left(i, j\right) = \argmin_{k: ~\mathbf{y}[k] \neq \mathbf{y}[i]}
D^2_{i,k} ~~\text{s.t.}~ D^2_{i,k} > D^2_{i,j}
\]
and $\mathcal{P}$ is the set of pairs with the same class label.
If there is  no such negative example satisfying the constraint, we just pick the furthest negative example
in the minibatch,
as follows:
\[
k^*\left(i, j\right) =
\argmax_{k: ~\mathbf{y}[k] \neq \mathbf{y}[i]} D^2_{i,k} 
\]

In order to get good results, the FaceNet paper had to use very large minibatches (1800 images),
to ensure they picked enough hard negatives. This makes it hard to train the model on a GPU due to the GPU memory constraint.
Below we describe some other losses which are easier to minimize using small minibatches.

\eat{
\begin{align}
\begin{split}
&\ell\left(X, \mathbf{y}\right) = \frac{1}{|\mathcal{P}|}\sum_{\left(i,j\right) \in \mathcal{P}} \left[ D_{i,j}^2 + \alpha - D^2_{i,k^*\left(i, j\right)} \right]_+\\
&\text{s.t.}~~~ I_{i,j} = \mathbb{I}\left[\exists k: \mathbf{y}[k] \neq \mathbf{y}[i], ~\text{s.t.}~ D^2_{i,k} > D^2_{i,j}\right] \\
&\hspace{1.8em} k^*\left(i, j\right) = I_{i,j} \left(\argmin_{k: ~\mathbf{y}[k] \neq \mathbf{y}[i]} D^2_{i,k} ~~\text{s.t.}~ D^2_{i,k} > D^2_{i,j}\right) +\\ 
&\hspace{6em} \left(1 - I_{i,j} \right) \left( \argmax_{k: ~\mathbf{y}[k] \neq \mathbf{y}[i]} D^2_{i,k} \right),
\end{split}
\end{align}

\noindent
where $\mathcal{P}$ denotes the set of pairs of examples with the same
class labels. The operator $[\cdot]_+$ denotes the hinge function
which takes the positive component of the argument,
$\mathbb{I}[\cdot]$ is the indicator function which returns $1$ if the
argument is true and $0$ otherwise and $D_{i,j}$ means the euclidean
distance of the two data points $X_i$ and $X_j$ in the embedding
space; $|| f(X_i) - f(X_j) ||_2$. The $\alpha$ denotes the fixed
margin constant. 
}

\subsection{Lifted structured embedding}

Song \textit{et al}. \cite{liftedstruct} proposed lifted structured embedding
 where each positive pair compares the distances
against all the negative pairs weighted by the margin constraint
violation. The idea is to have a differentiable smooth loss which
incorporates the online hard negative mining functionality using the
\emph{log-sum-exp} formulation. 

\begin{align}
\begin{split}
\ell\left(X, \mathbf{y}\right) = \frac{1}{2|\mathcal{P}|}\sum_{\left(i,j\right) \in \mathcal{P}}& \Bigg[ \log \Bigg( \sum_{\left(i,k\right) \in \mathcal{N}} \exp \left\{\alpha - D_{i,k}\right\} +\\
&\sum_{\left(j,l\right) \in \mathcal{N}} \exp \left\{ \alpha - D_{j,l} \right\} \Bigg) + D_{i,j} \Bigg]_+^2,
\end{split}
\end{align}

\noindent
where $\mathcal{N}$ denotes the set of pairs of examples with different class labels.

\subsection{N-pairs embedding}

Recently, Sohn \textit{et al}. \cite{npairs} proposed N-pairs loss 
which enforces softmax cross-entropy loss among the pairwise
similarity values in the batch. 

\begin{align}
\begin{split}
  \ell\left(X, \mathbf{y}\right) = \frac{-1}{|\mathcal{P}|}
  &\sum_{(i,j) \in \mathcal{P}} \log \frac{\exp\{S_{i,j}\}}{\exp\{S_{i,j}\} + \displaystyle\sum_{k: ~\mathbf{y}[k] \neq \mathbf{y}[i]} \exp\{S_{i,k}\}}\\
  &+ \frac{\lambda}{m} \sum_i^m ~ || f(X_i) ||_2,
\end{split}
\end{align}

\noindent where $S_{i,j}$ means the feature dot product between two data points
in the embedding space; $S_{i,j} = f(X_i)^\intercal f(X_j)$, $m$ is the number of the data, and $\lambda$ is the regularization constant for the $\ell_2$ regularizer on the embedding vectors. 

\subsection{Other related work}

In addition to the above work on metric learning,
there has been some recent work on learning to cluster
with deep networks.
Hershey \textit{et al}. \cite{deepclustering} uses Frobenius norm
on the residual between the binary ground truth and the estimated pairwise
affinity matrix;
they apply this to speech spectrogram signal clustering.
However,  using the Frobenius norm directly is
suboptimal, since it ignores the fact that the affinity matrix is
positive definite.

To overcome this, matrix backpropagation
\cite{matrixbackprop} first projects the true and predicted affinity
matrix to a metric space where Euclidean distance is appropriate.
Then it applies this to normalized cuts for unsupervised image segmentation.
However, this approach
requires computing the eigenvalue decomposition of the data matrix,
which has cubic time complexity in the number of data and is thus not very practical for
large problems.

\section{Methods}
\label{sec:methods}

One of the key attributes which the recent state of the art deep
learning methods  in Section~\ref{sec:related} have in common is
that they are all local metric learning
methods. Figure~\ref{fig:local_failure}
illustrates a case where this can fail.
In particular,
whenever a positive pair (such as the two purple bold
dots connected by the blue edge)
is separated by examples from other classes, the attractive gradient
signal from the positive pair gets outweighed by the repulsive gradient
signal from the negative data points (yellow and green data points
connected with the red edges).
This failure can lead to groups of examples with the same class labels
being separated into partitions in the embedding space that are far
apart from each other. This can lead to degradation in
the clustering and nearest neighbor based retrieval performance.
For example, suppose we incorrectly created 4 clusters in
Figure~\ref{fig:local_failure}.
If we asked for the  $12$ nearest neighbors
of one of purple points, we would retrieve points belonging to other
classes.

To overcome this problem, we propose a method that learns to embed
points so as to minimize a clustering loss, as we describe below.

\eat{
"Standard methods for learning embeddings measure the local similarity
between pairs, or triples, or quads of points. In this paper, we
propose a new perspective. Specifically, suppose we use the embedding
to assign each point to its most similar neighbor; we then measure the
resulting "distortion" (sum of distances to each centroid). We propose
to compare this to the distortion induced by the optimal clustering
based on the ground truth labels. This gives us a more global measure
of quality of an embedding. We show how to optimize the difference
between the optimal and measured distortion, plus a (data-dependent)
margin term, using a greedy optimization procedure inside of the SGD
inner loop."
}

\begin{figure}[thbp]
\centering
\includegraphics[width=0.40\textwidth]{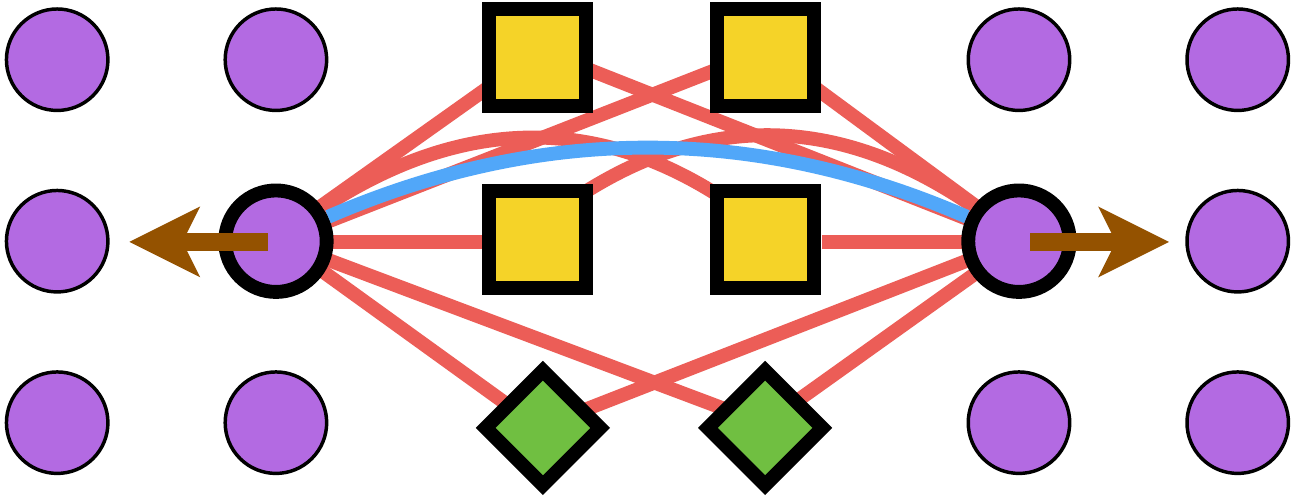}
\caption{Example failure mode for local metric learning
  methods. Whenever a positive pair (linked with the blue edge)
  is separated by negative examples, the gradient signal from the
  positive pair (attraction) gets outweighed by the negative pairs
  (repulsion). Illustration shows the failure case for 2D embedding
  where the purple clusters can't be merged into one cluster.} 
\label{fig:local_failure}
\end{figure}

\begin{figure}[thbp]
\centering
\includegraphics[width=0.46\textwidth]{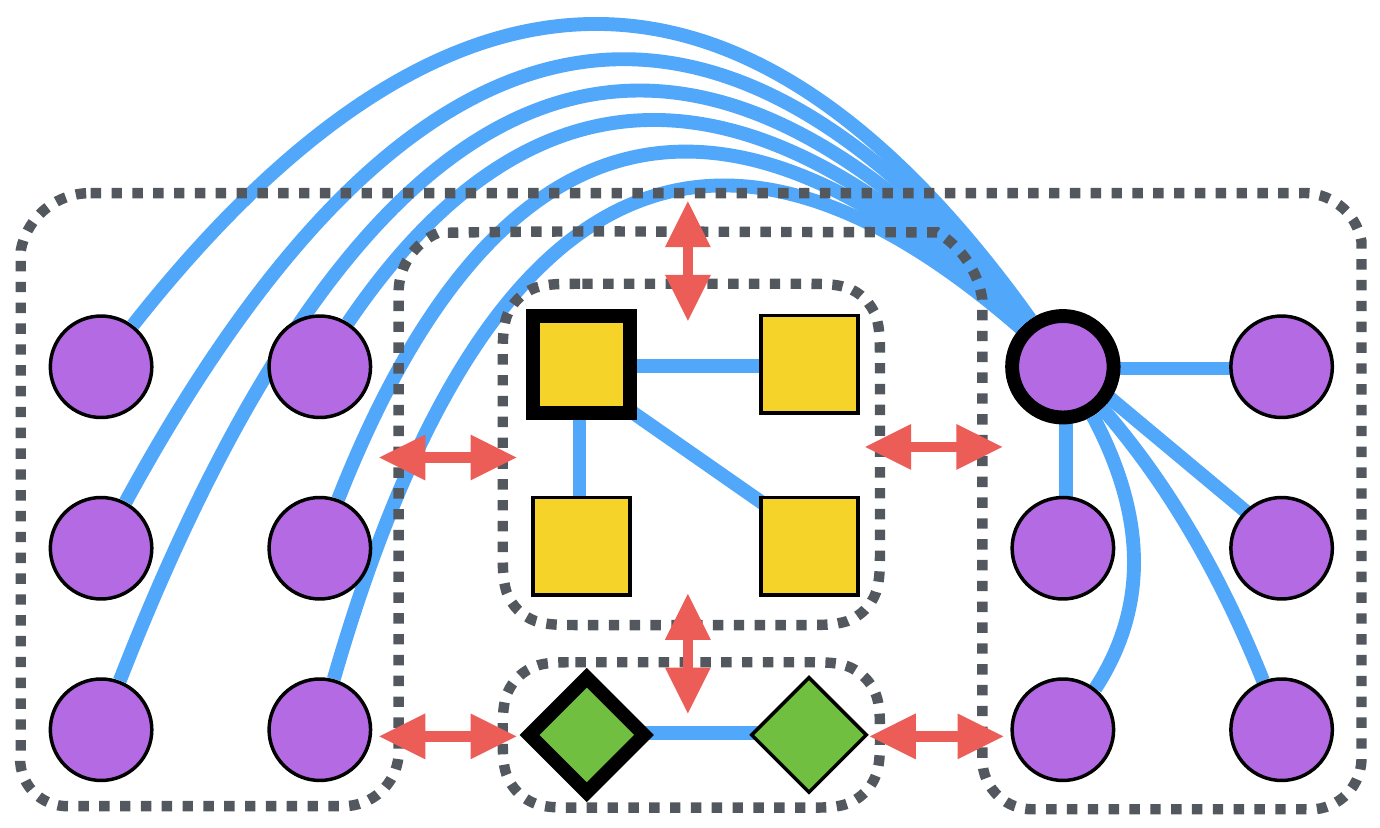}
\caption{Proposed clustering loss for the same embedding layout in
  figure \ref{fig:local_failure}. Nodes highlighted in bold are the
  cluster medoids. The proposed method encourages small sum of
  distances within each cluster, while discouraging different clusters from
  getting close to each other.
} 
\label{fig:clustering_loss}
\end{figure}

\subsection{Facility location problem}

Suppose we have a set of inputs $X_i$,
and an embedding function $f(X_i;\Theta)$ that maps each input to a
point in some $K$ dimensional space.
Now suppose we compress this set of points by mapping each example
$i$ to its nearest point from a chosen set of landmarks
$S \subseteq \mathcal{V}$,
where  $\mathcal{V}  = \{1, \ldots, |X|\}$ is the ground set.
We can define the resulting function as  follows:
\begin{align}
F(X, S; \Theta) = -\sum_{i \in |X|} \min_{j \in S} ||f(X_i; \Theta) - f(X_j; \Theta)||,
\label{eqn:facility}
\end{align}
This is called the
facility location function,
and has been widely used in data summarization
and clustering \cite{bilmes12, bilmes14}.
The idea is that this
function measures the sum of the travel distance for each customer in
$X$ to their respective nearest facility location in $S$. In terms of
clustering, data points in $S$ correspond to the cluster medoids, and
the cluster assignment is based on the nearest medoid from each data
point. Maximizing equation \ref{eqn:facility} with respect to subset
$S$ is NP-hard, but there is a well established worst case optimality
bound of $O\left(1 - \frac{1}{e}\right)$ for the greedy solution of
the problem via submodularity \cite{krause2012submodular}.

Below we show how to use the facility location problem as a subroutine
for deep metric learning.

\subsection{Structured facility location for deep metric learning}

The oracle scoring function $\tilde{F}$ measures the quality of the
clustering given the ground truth clustering assignment $\mathbf{y}^*$ and
the embedding parameters $\Theta$:

\begin{align}
\tilde{F}(X, \mathbf{y}^*; \Theta) = \sum_{k}^{|\mathcal{Y}|} \max_{j \in \left\{i:~ \mathbf{y}^*[i] = k\right\}} F\left(X_{\left\{i:~ \mathbf{y}^*[i] = k\right\}}, \{j\}; \Theta \right),
\end{align}

\noindent where $\left\{i:~ \mathbf{y}^*[i] = k\right\}$ denotes the subset of the elements in $\mathcal{V}$ with the ground truth label equal to $k$.

We would like the clustering score of the oracle clustering assignment to be greater than the score for the maximally violating clustering assignment. Hence we define the following structured loss function:


\begin{align}
\begin{split}
\ell\left(X, ~\mathbf{y}^{*}\right) = \Bigg[ &\underbrace{\max_{\substack{S \subset \mathcal{V} \\|S|=|\mathcal{Y}|}} \bigg\{ F(X, S; ~\Theta) + \gamma ~\Delta\left(g(S), ~\mathbf{y}^*\right)}_{\left(\ast\right)} \bigg\} \\
& - \tilde{F}(X, \mathbf{y}^*; ~\Theta) \Bigg]_+
\end{split}
\label{eqn:clustering_loss}
\end{align}

We will define the structured margin
$\Delta\left(\mathbf{y}, ~\mathbf{y}^*\right)$ below.
The function $y = g(S)$  maps the set of indices $S$
to a set of cluster labels by
assigning each  data point to its nearest facility in
$S$:

\begin{align}
g(S)[i] = \displaystyle \argmin_{j} ||f(X_i; \Theta) - f(X_{\left\{j |~ j \in S\right\}}; \Theta)||
\end{align}

Intuitively, the loss function in Equation \ref{eqn:clustering_loss}
encourages the network to learn an  embedding function
$f\left(\cdot; \Theta \right)$ such that the oracle clustering score
$\tilde{F}$ is greater than the clustering score $F$ for any other
cluster assignments $g(S)$ at least by the structured margin
$\Delta\left(\mathbf{y}, ~\mathbf{y}^*\right)$. Figure \ref{fig:figure1} gives the pictorial illustration of the overall framework.

The structured margin term $\Delta\left(\mathbf{y},
~\mathbf{y}^*\right)$ measures the quality of the clustering. The
margin term outputs $0$ if the clustering quality of $\mathbf{y}$ with
respect to the ground truth clustering assignment $\mathbf{y}^*$ is
perfect (up to a permutation) and $1$ if the quality is the worst. We
use the following margin term

\begin{align}
\Delta\left(\mathbf{y}, ~\mathbf{y}^*\right) = 1-\text{NMI}\left(\mathbf{y}, ~\mathbf{y}^*\right)
\label{eqn:structured_loss}
\end{align}

\noindent
where NMI is the normalized mutual information (NMI)
\cite{manningbook}.
This measures the label agreement between the two clustering
assignments ignoring the permutations. It is defined by the ratio of
mutual information and the square root of the product of entropies for each assignments:
\begin{align}
NMI(\mathbf{y}_1,
\mathbf{y}_2)=\frac{MI(\mathbf{y}_1,\mathbf{y}_2)}{\sqrt{H(\mathbf{y}_1)
  H(\mathbf{y}_2)}}
\end{align}
The marginal and joint probability mass used for
computing the entropy and the mutual information can be estimated as follows:

\begin{align}
    \begin{split}
    &P(i)=\frac{1}{m}~\sum_{j} \mathbb{I}[\mathbf{y}[j]==i]\\
    &P(i,j)=\frac{1}{m}~\sum_{k,l} \mathbb{I}[\mathbf{y}_1[k]==i]\cdot \mathbb{I}[\mathbf{y}_2[l]==j],
    \end{split}
\end{align}

\noindent where $m$ denotes the number of data (also equal to $|X|$).

Figure \ref{fig:clustering_loss} illustrates the advantages of the
proposed algorithm. Since the algorithm is aware of the global
landscape of the embedding space, it can overcome the bad local optima
in figure \ref{fig:local_failure}. The clustering loss encourages
small intra cluster (outlined by the dotted lines in figure
\ref{fig:clustering_loss}) sum of distances with respect to each
cluster medoids (three data points outlined in bold) while
discouraging different clusters from getting close to each other via
the NMI metric in the structured margin term.  

\subsection{Backpropagation subgradients}

We fit our model using stochastic gradient descent.
The key step is to compute the derivative of the loss, which is given
by the following expression:

\begin{align}
\begin{split}
  \partial ~\ell\left(X, \mathbf{y}^*\right) = ~&
  \mathbb{I}\left[\ell\left(X, \mathbf{y}^*\right) > 0\right] \Big(\nabla_{\Theta} F\left(X, S_{\text{PAM}}; \Theta \right)\\
& - \nabla_{\Theta} \tilde{F}\left(X, \mathbf{y}^*; \Theta\right) \Big)\\
\end{split}
\label{eqn:backprop}
\end{align}

\noindent
Here  $S_{\text{PAM}}$ is the solution to the
subproblem marked $(\ast)$ in Equation~\ref{eqn:clustering_loss};
we discuss how to compute this in Section~\ref{sec:PAM}.

The first gradient term is as follows:

\makeatletter
\newcommand*\bigcdot{\mathpalette\bigcdot@{.4}}
\newcommand*\bigcdot@[2]{\mathbin{\vcenter{\hbox{\scalebox{#2}{$\m@th#1\bullet$}}}}}
\makeatother

\begin{align}
  \nabla_{\Theta} F\left(X, S; \Theta\right)
  &= -\sum_{i \in |X|} \Bigg[
    \frac{f(X_i; \Theta) - f(X_{j^*(i)}; \Theta)}
         {|| f(X_i; \Theta) - f(X_{j^*(i)}; \Theta) ||}  \notag\\ 
& \bigcdot \nabla_{\Theta} \left( f(X_i; \Theta) - f(X_{j^*(i)}; \Theta)
         \right)\Bigg]
\end{align}

\noindent where $j^*(i)$ denotes the index of the closest facility
location in the set $S_{\text{PAM}}$. The gradient for the oracle
scoring function can be derived by computing
\begin{align}
  \nabla_{\Theta}
\tilde{F}\left(X, \mathbf{y}_i^*; \Theta\right) = \sum_{k}
\nabla_{\Theta} F\left(X_{\left\{i:~ \mathbf{y}^*[i] = k\right\}},
\{j^*(k)\}; \Theta \right)
\end{align}

Equation \ref{eqn:backprop} is the formula for the exact subgradient and we find an approximate maximizer $S_{\text{PAM}}$ in the equation (section~\ref{sec:PAM}) so we have an approximate subgradient. However, this approximation works well in practice and have been used for structured prediction setting \cite{bilmes12, bilmes14}.

\subsection{Loss augmented inference}
\label{sec:PAM}

We solve the optimization problem $(\ast)$ in Equation~\ref{eqn:clustering_loss} in two steps. First, we use the greedy Algorithm~\ref{alg:greedy} to select an initial good set of facilities. In each step, it chooses the element $i^*$ with the best marginal benefit. The running time of the algorithm is
$O\left(|\mathcal{Y}|^3 \cdot |\mathcal{V}|\right)$,
where $|\mathcal{Y}|$ denotes the number of clusters in the batch and $\mathcal{V}  = \{1, \ldots, |X|\}$. This time is linear in the size of the minibatch, and hence does not add much overhead on top of the gradient computation.
Yet, if need be, we can speed up this part via a stochastic version of the greedy algorithm \cite{mirzasoleiman15}. 

This algorithm is motivated by the fact that the first term, $F(X,S; \Theta)$, is a monotone submodular function in $S$. We observed that throughout the learning process, this term is large compared to the second, margin term. Hence, in this case, our function is still close to submodular. For approximately submodular functions, the greedy algorithm can still be guaranteed to work well \cite{das11}.


\begin{algorithm}
    \SetKwInOut{Input}{Input}
    \SetKwInOut{Initialize}{Initialize}
    \SetKwInOut{Output}{Output}
    \SetKwInOut{Define}{Define}

    \Input{$X \in \mathbb{R}^{m\times d},~ \mathbf{y}^* \in |\mathcal{Y}|^m,~ \gamma$}
    \Output{$S \subseteq \mathcal{V}$}
    \Initialize{$S = \left\{\emptyset\right\}$}
    \Define{$A(S) := F(X, S; \Theta) + \gamma \Delta\left(g(S), ~\mathbf{y}^*\right)$}
    \vspace{1.0em}
    
    \While{$|S| < |\mathcal{Y}|$}
    {
    	$i^* = \displaystyle \argmax_{i \subseteq \mathcal{V} \setminus S} ~ A\left(S \cup \{i\}\right) - A(S)$\\
	\vspace{0.5em}
	$S := S \cup \{ i^* \}$
    }
    \Return S
    \caption{Loss augmented inference for ($\ast$)}
    \label{alg:greedy}
\end{algorithm}

Yet, since $A(S)$ is not entirely submodular, we refine the greedy solution with a local search, Algorithm \ref{alg:pam}. This algorithm performs pairwise exchanges of current medoids $S[k]$ with alternative points $j$ in the same cluster.
The running time of the algorithm is
$O\left(T |\mathcal{Y}|^3 \cdot |\mathcal{V}|\right)$,
where $T$ is the maximum number of iterations.
In practice, it converges quickly, so
we run the algorithm for $T=5$ iterations only.

Algorithm \ref{alg:pam} is similar to
the partition around medoids (PAM) \cite{pam} algorithm for k-medoids
clustering, which independently reasons about each cluster during the
medoid swapping step.
Algorithm~\ref{alg:pam} differs from PAM by the structured margin term, which involves all clusters simultaneously.

The following lemma states that the algorithm can only improve over the greedy solution:
\begin{lemma}
  Algorithm~\ref{alg:pam} monotonically increases the objective function $A(S) = F(X, S; \Theta) + \gamma \Delta\left(g(S), ~\mathbf{y}^*\right)$.
\end{lemma}
\begin{proof}
  In any step $t$ and for any $k$, let $c = S[k]$ be the $k$th medoid in $S$.
  The algorithm finds the point $j$ in the $k$th cluster such that $A((S \setminus \{c\})\cup\{j\})$ is maximized. Let $j^*$ be a maximizing argument. Since $j = c$ is a valid choice, we have that $A((S \setminus \{c\})\cup\{j^*\}) \geq A((S \setminus \{c\})\cup\{c\} ) = A(S)$, and hence the value of $A(S)$ can only increase.
\end{proof}

In fact, with a small modification and $T$ large enough, the algorithm is guaranteed to find a local optimum, i.e., a set $S$ such that $A(S) \geq A(S')$ for all $S'$ with $|S \Delta S'|=1$ (Hamming distance one). Note that the overall problem is NP-hard, so a guarantee of global optimality is impossible.
\begin{lemma}
  If the exchange point $j$ is chosen from $X$ and $T$ is large enough that the algorithm terminates because it makes no more changes, then Algorithm~\ref{alg:pam} is guaranteed to find a local optimum.
\end{lemma}


\begin{algorithm}
    \SetKwInOut{Input}{Input}
    \SetKwInOut{Initialize}{Initialize}
    \SetKwInOut{Output}{Output}
    \SetKwInOut{Define}{Define}

    \Input{$X \in \mathbb{R}^{m\times d},~ \mathbf{y}^* \in
      |\mathcal{Y}|^m,~ S_{\text{init}}, ~\gamma, T$}
    \Output{$S$}
    \Initialize{$S = S_{\text{init}}, t=0$}
    \vspace{1.0em}
    \For{$t < T$}
    {
    	\small{\tcp{Perform cluster assignment}}
	$\mathbf{y}_{\text{PAM}} = g\left(S\right)$\\
	\vspace{1em}
	\small{\tcp{Update each medoids per cluster}}
	\For{$k < |\mathcal{Y}|$}{
		\small{\tcp{Swap the current medoid in cluster k if it increases the score.}}
		$S[k] = \displaystyle \argmax_{j \in \left\{i:~ \mathbf{y}_{\text{PAM}}[i] = k\right\}} F\left(X_{\left\{i:~ \mathbf{y}_{\text{PAM}}[i] = k\right\}}, \{j\}; \Theta \right)$\\
		\vspace{0.5em}
		$~~~~~~~~~~~~~~~~~ + \gamma \Delta \left(g\left(S \setminus \left\{S[k]\right\} \cup \left\{j \right\} \right), ~\mathbf{y}^* \right)$
	}
    }
    \Return S
    \caption{Loss augmented refinement for ($\ast$)}
    \label{alg:pam}
\end{algorithm}

\subsection{Implementation details}
\label{sec:implementation}

We used Tensorflow \cite{tensorflow} package for our
implementation. For the embedding vector, we $\ell_2$ normalize the embedding vectors before
computing the loss for our method. The model slightly underperformed when we omitted the embedding normalization. We also tried solving the loss augmented
inference using Algorithm \ref{alg:pam} with random initialization, but
it didn't work as well as initializing the algorithm with the greedy
solution from Algorithm \ref{alg:greedy}.

For the network architectures, we used the Inception \cite{inceptionv1} network with
batch normalization \cite{batchnorm} pretrained on ILSVRC 2012-CLS
\cite{imagenet} and finetuned the network on our datasets. All the
input images are first resized to square size ($256 \times 256$)
and cropped at $227 \times 227$. For the data augmentation, we used
random crop with random horizontal mirroring for training and a single
center crop for testing. In Npairs embedding \cite{npairs}, they take multiple random crops and average the embedding vectors from the cropped images during testing. However, in our implementation of \cite{npairs}, we take a single center crop during testing for fair comparison with other methods. 

The experimental ablation study reported in
\cite{liftedstruct} suggested that the embedding size doesn't play a
crucial role during training and testing phase so we decided to fix
the embedding size at $d=64$ throughout the experiment (In \cite{liftedstruct}, the authors report the recall@K results with $d=512$ and provided the results for $d=64$ to us for fair comparison). We used RMSprop
\cite{rmsprop} optimizer with the batch size $m$ set to $128$. For the margin multiplier constant $\gamma$, we gradually decrease it using exponential decay with the decay rate set to $0.94$.

As briefly mentioned in section \ref{sec:intro}, the proposed method does not require the data to be prepared in any rigid paired format (pairs, triplets, n-pair tuples, etc). Instead we simply sample $m$ (batch size) examples and labels at random. That said, the clustering loss becomes trivial if a batch of data all have the same class labels (perfect clustering merging everything into one cluster) or if the data all have different class labels (perfect clustering where each data point forms their own clusters). In this regard, we guarded against those pathological cases by ensuring the number of unique classes ($C$) in the batch is within a reasonable range. We tried three different settings $\frac{C}{m} = \{0.25, 0.50, 0.75\}$ and the choice of the ratio did not lead to significant changes in the experimental results. For the CUB-200-2011 \cite{birds} and Cars196 \cite{cars}, we set $\frac{C}{m} = 0.25$. For the Stanford Online Products \cite{liftedstruct} dataset, $\frac{C}{m} = 0.75$ was the only possible choice because the dataset is extremely fine-grained.

\section{Experimental results}
\label{sec:experiments}
\label{sec:results}

Following the experimental protocol in \cite{liftedstruct, npairs}, we
evaluate the clustering and $k$ nearest neighbor retrieval
\cite{recall_at_K} results on data from previously unseen classes on
the CUB-200-2011 \cite{birds}, Cars196 \cite{cars}, and Stanford
Online Products \cite{liftedstruct} datasets. We compare our method
with three current state of the art methods in deep metric learning:
(1) triplet learning with semi-hard negative mining strategy
\cite{facenet}, (2) lifted structured embedding \cite{liftedstruct},
(3) N-pairs metric loss \cite{npairs}. To be comparable with prior
work,
we
$\ell_2$ normalize the embedding for the triplet (as prescribed by
\cite{facenet}) and our method, but not for the lifted structured loss and
the N-pairs loss (as in the implementation sections in
\cite{liftedstruct, npairs}). 

We used the same train/test split as in \cite{liftedstruct} for all
the datasets. The CUB200-2011 dataset \cite{birds} has $11,788$ images
of $200$ bird species;  we used the first $100$ birds species for
training and the remaining $100$ species for testing. The Cars196
dataset \cite{cars} has $16,185$ images of $196$ car models. We used
the first $98$ classes of cars for training and the rest  for
testing. The Stanford online products dataset \cite{liftedstruct} has
$120,053$ images of $22,634$ products sold online on eBay.com. We used
the first $11,318$ product categories for training and the remaining
$11,316$ categories for testing. 

\subsection{Quantitative results}

The training procedure for all the methods converged at $10k$ iterations for the CUB200-2011 \cite{birds} and at $20k$ iterations for the Cars196 \cite{cars} and the Stanford online products \cite{liftedstruct} datasets.

Tables \ref{tab:birds}, \ref{tab:cars}, and \ref{tab:products} shows
the results of the quantitative comparison between our
method and other deep metric learning methods.
We report the NMI score, to measure the quality of the clustering,
as well as  $k$ nearest
neighbor performance with the Recall@K metric. The tables show that our proposed method has the state of the art performance on both the NMI and R@K metrics outperforming all the previous methods.
  
\begin{table}[htbp]
\centering
\renewcommand{\arraystretch}{1.2}
\renewcommand{\tabcolsep}{1.5mm}
\begin{tabular}{*{6}{c}}
\toprule
& NMI & R@1 & R@2 & R@4 & R@8\\
\midrule
\midrule
Triplet semihard \cite{facenet}   & 55.38 & 42.59 & 55.03 & 66.44 &  77.23 \\
Lifted struct \cite{liftedstruct}     & 56.50 & 43.57 & 56.55 & 68.59 & 79.63 \\
Npairs \cite{npairs} & 57.24 & 45.37 & 58.41 & 69.51 & 79.49 \\
Clustering (Ours)  & \bf{59.23}  & \bf{48.18} & \bf{61.44} & \bf{71.83} & \bf{81.92} \\
\bottomrule
\end{tabular}
\vspace{0.5em}
\caption{Clustering and recall performance on CUB-200-2011 \cite{birds} @$10k$ iterations.}
\label{tab:birds}
\end{table}

\begin{table}[htbp]
\centering
\renewcommand{\arraystretch}{1.2}
\renewcommand{\tabcolsep}{1.5mm}
\begin{tabular}{*{6}{c}}
\toprule
& NMI & R@1 & R@2 & R@4 & R@8\\
\midrule
\midrule
Triplet semihard \cite{facenet}   & 53.35 & 51.54 & 63.78 & 73.52 & 82.41\\
Lifted struct \cite{liftedstruct}     & 56.88 & 52.98 & 65.70 & 76.01 & 84.27 \\
Npairs \cite{npairs} & 57.79 & 53.90 & 66.76 & 77.75 & 86.35\\
Clustering (Ours) & \bf{59.04} & \bf{58.11} & \bf{70.64} & \bf{80.27} & \bf{87.81}\\
\bottomrule
\end{tabular}
\vspace{0.5em}
\caption{Clustering and recall performance on Cars196 \cite{cars} @$20k$ iterations.}
\label{tab:cars}
\end{table}


\begin{table}[htbp]
\centering
\renewcommand{\arraystretch}{1.2}
\renewcommand{\tabcolsep}{2.0mm}
\begin{tabular}{*{6}{c}}
\toprule
& NMI & R@1 & R@10 & R@100 \\
\midrule
\midrule
Triplet semihard \cite{facenet}   & 89.46 & 66.67 & 82.39 & 91.85 \\
Lifted struct \cite{liftedstruct}     & 88.65 & 62.46 & 80.81 & 91.93\\
Npairs \cite{npairs} & 89.37 & 66.41 & 83.24 & 93.00 \\
Clustering (Ours)  & \bf{89.48} & \bf{67.02} & \bf{83.65} & \bf{93.23} \\
\bottomrule
\end{tabular}
\vspace{0.5em}
\caption{Clustering and recall performance on Products \cite{liftedstruct} @$20k$ iterations.}
\label{tab:products}
\end{table}

\subsection{Qualitative results}

Figure \ref{fig:tsne-birds}, \ref{fig:tsne-cars}, and
\ref{fig:tsne-products} visualizes the t-SNE \cite{tsne} plots
on the embedding vectors from our method on CUB200-2011 \cite{birds},
Cars196 \cite{cars}, and Stanford online products \cite{liftedstruct}
datasets respectively. The plots are best viewed on a monitor when
zoomed in. We can see that our embedding does a great job on grouping
similar objects/products despite the significant variations in view
point, pose, and configuration.

\begin{figure}[thbp]
\centering
\includegraphics[width=0.499\textwidth]{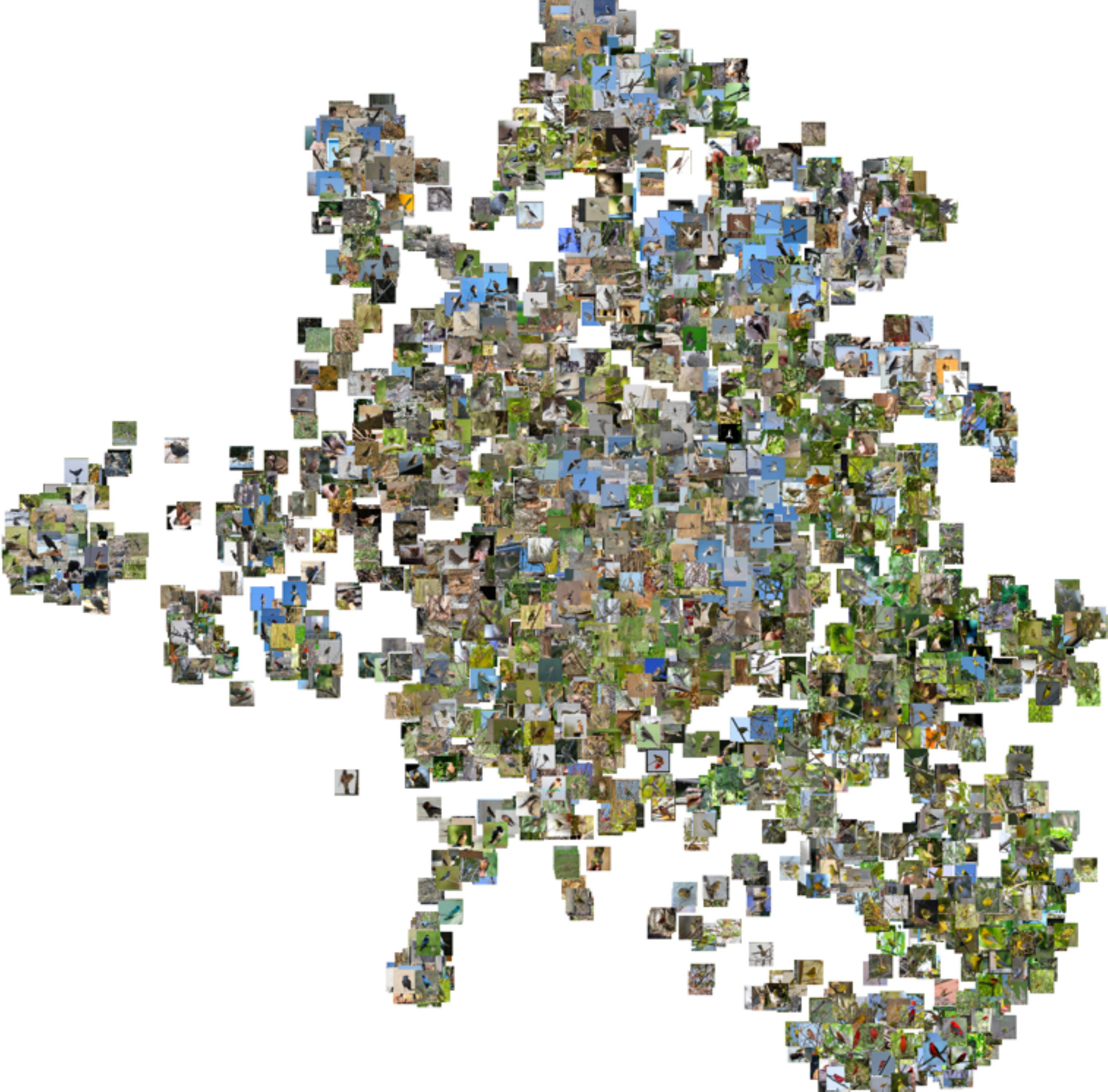}
\caption{Barnes-Hut t-SNE visualization \cite{tsne} of our embedding on the  CUB-200-2011 \cite{birds} dataset. Best viewed on a monitor when zoomed in.}
\label{fig:tsne-birds}
\end{figure}

\begin{figure}[thbp]
\centering
\includegraphics[width=0.499\textwidth]{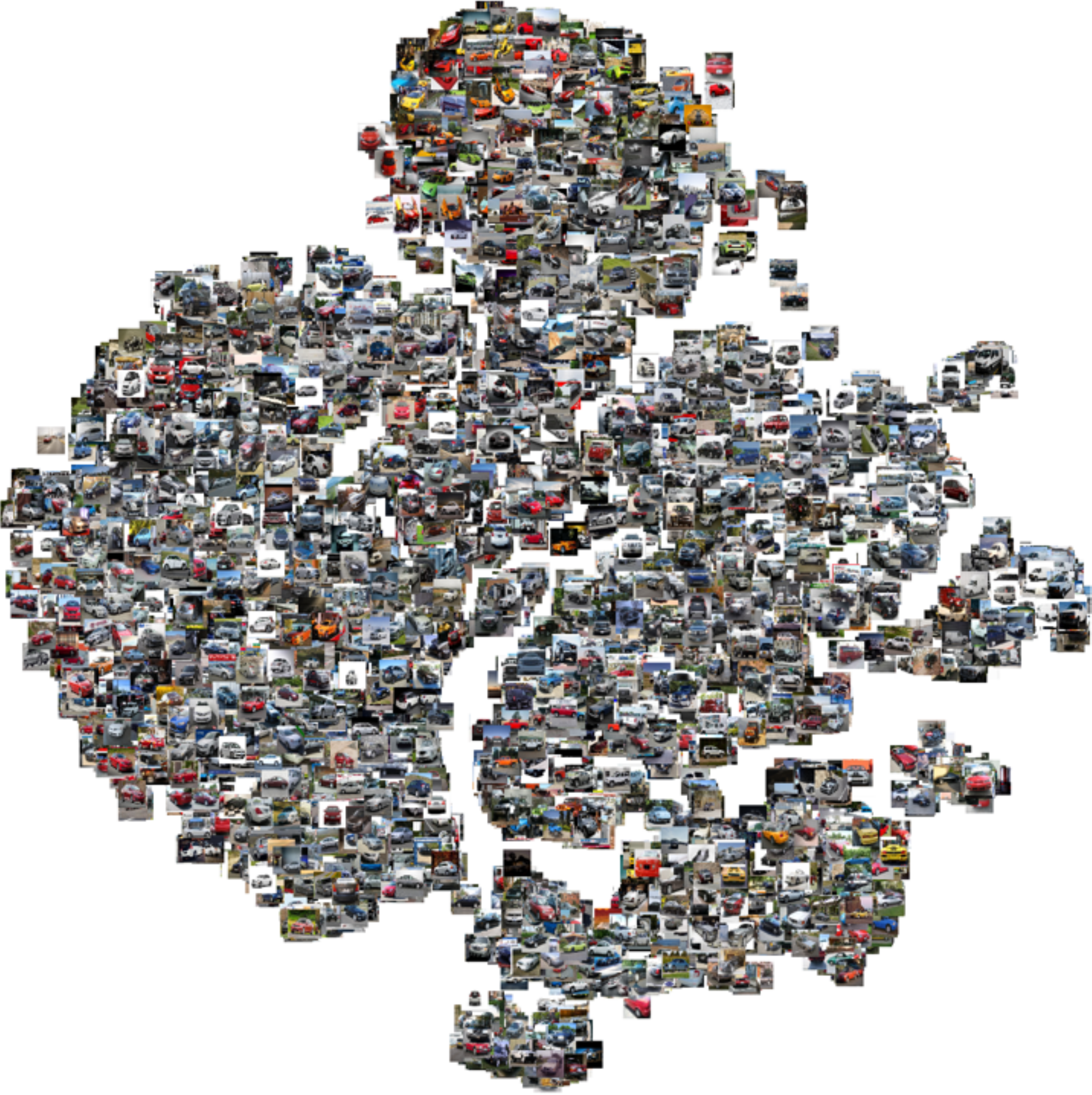}
\caption{Barnes-Hut t-SNE visualization \cite{tsne} of our embedding on the Cars196 \cite{cars} dataset. Best viewed on a monitor when zoomed in.}
\label{fig:tsne-cars}
\end{figure}

%

\begin{figure*}[thbp]
\centering
\includegraphics[width=\textwidth]{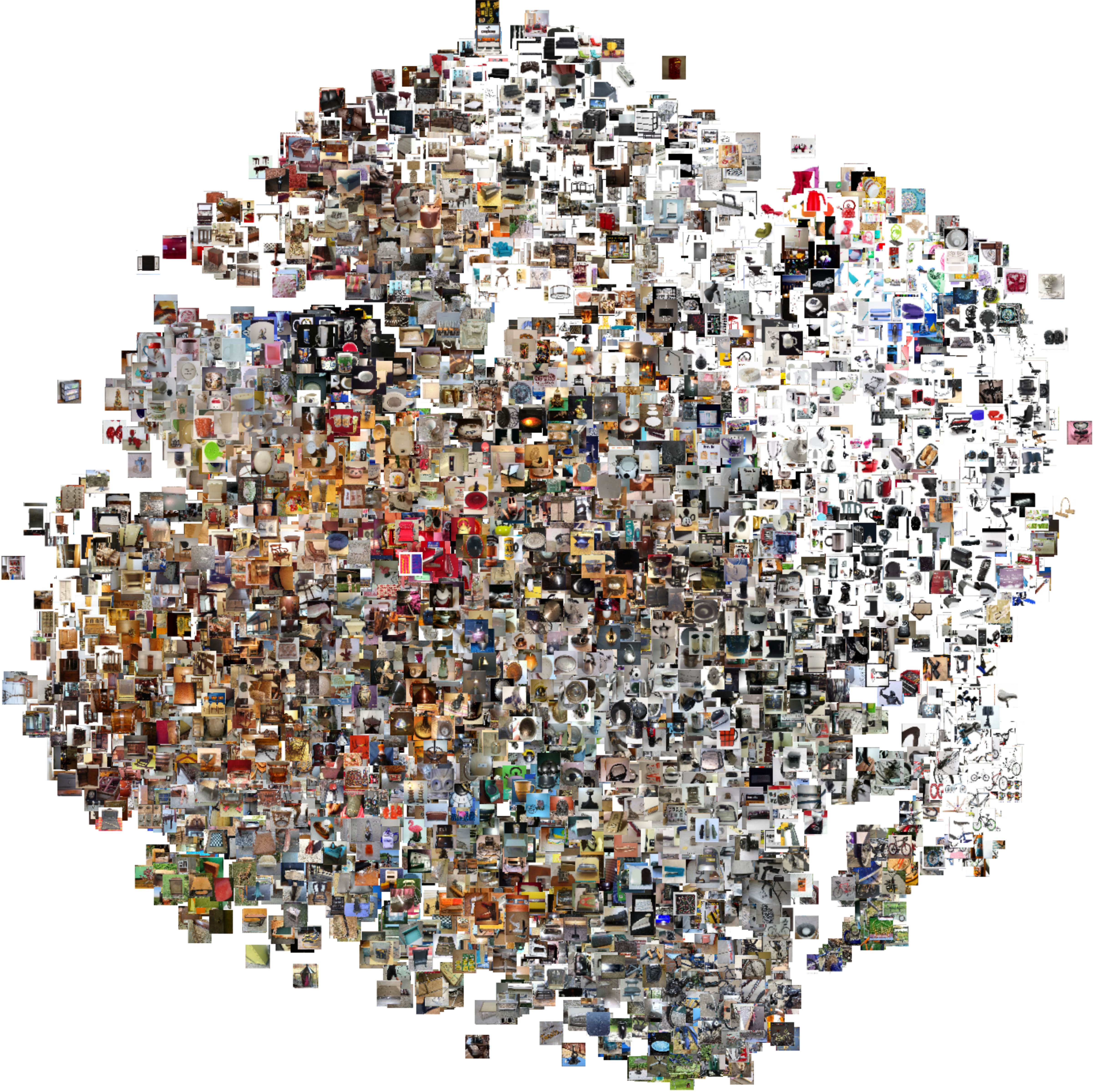}
\caption{Barnes-Hut t-SNE visualization \cite{tsne} of our embedding on the Stanford online products dataset \cite{liftedstruct}. Best viewed on a monitor when zoomed in.}
\label{fig:tsne-products}
\end{figure*}


\section{Conclusion}
\label{sec:conclusion}

We described a novel learning scheme for optimizing the deep metric embedding with the learnable clustering function and the clustering metric (NMI) in a end-to-end fashion within a principled structured prediction framework.

Our experiments on CUB200-2011 \cite{birds}, Cars196 \cite{cars}, and Stanford online products \cite{liftedstruct} datasets show state of the art performance both on the clustering and retrieval tasks.

The proposed clustering loss has the added benefit that it doesn't require rigid and time consuming data preparation (i.e. no need for preparing the data in pairs \cite{contrastive}, triplets \cite{triplet, facenet}, or n-pair tuples \cite{npairs} format). This characteristic of the proposed method opens up a rich class of possibilities for advanced data sampling schemes.

In the future, we plan to explore sampling based gradient averaging scheme where we ask the algorithm to cluster several random subsets of the data within the training batch and then average the loss gradient from multiple sampled subsets in similar spirit to Bag of Little Bootstraps (BLB) \cite{blb}. 

{\small
\bibliographystyle{ieee}
\bibliography{cvpr17}
}

\end{document}